\documentclass[10pt,twocolumn,letterpaper]{article}

\usepackage{iccv}
\usepackage{times}
\usepackage{epsfig}
\usepackage{graphicx}
\usepackage{amsmath}
\usepackage{amssymb}

\usepackage{amsmath, amsfonts, amssymb, amsthm}
\usepackage{graphicx}
\usepackage{framed}
\usepackage{bbm}
\usepackage{verbatim}
\usepackage{color}
\usepackage{multirow, tabularx}
\usepackage{soul}
\usepackage{subcaption}
\usepackage{float}
\usepackage{makecell}
\usepackage{longtable}
\usepackage{enumitem}

\usepackage{mystyle}

\usepackage[pagebackref=true,breaklinks=true,letterpaper=true,colorlinks,bookmarks=false]{hyperref}

\iccvfinalcopy

\ificcvfinal\fi
\begin{document}

\title{Learning Motion in Feature Space: Locally-Consistent Deformable Convolution Networks for Fine-Grained Action Detection}

\author{
    Khoi-Nguyen C. Mac$^1$, Dhiraj Joshi$^2$, Raymond A. Yeh$^1$, Jinjun Xiong$^2$, Rogerio S. Feris$^2$, Minh N. Do$^1$\\
	$^1$University of Illinois at Urbana-Champaign, $^2$IBM Research AI \\
	$^1${\tt\small \{knmac, yeh17, minhdo\}@illinois.edu}, $^2${\tt\small \{djoshi, jinjun, rsferis\}@us.ibm.com}
}

\maketitle

\begin{abstract}
Fine-grained action detection is an important task with numerous applications in robotics and human-computer interaction. Existing methods typically utilize a two-stage approach including extraction of local spatio-temporal features followed by temporal modeling to capture long-term dependencies. While most recent papers have focused on the latter (long-temporal modeling), here, we focus on producing features capable of modeling fine-grained motion more efficiently. We propose a novel locally-consistent deformable convolution, which utilizes the change in receptive fields and enforces a local coherency constraint to capture motion information effectively. Our model jointly learns spatio-temporal features (instead of using independent spatial and temporal streams). The temporal component is learned from the feature space instead of pixel space, \textit{e.g.} optical flow. The produced features can be flexibly used in conjunction with other long-temporal modeling networks, \textit{e.g.} ST-CNN, DilatedTCN, and ED-TCN. Overall, our proposed approach robustly outperforms the original long-temporal models on two fine-grained action datasets: 50 Salads and GTEA, achieving F1 scores of 80.22\% and 75.39\% respectively. Source code is available at: \url{https://github.com/knmac/LCDC_release}.
\end{abstract}

\section{Introduction}
\label{sec:intro}

\begin{figure}
	\centering
	\begin{subfigure}[t]{0.45\linewidth}
		\includegraphics[width=\linewidth]{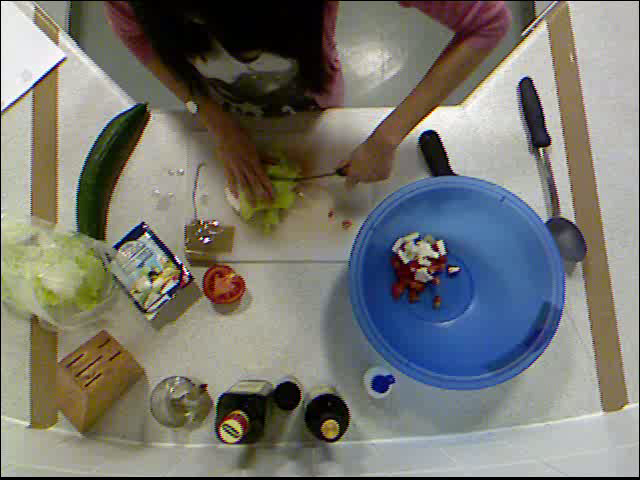}
		\caption{frame at time $t$-1.}
		\label{sfig:compare_optical_flow_f1}
	\end{subfigure}
	\begin{subfigure}[t]{0.45\linewidth}
		\includegraphics[width=\linewidth]{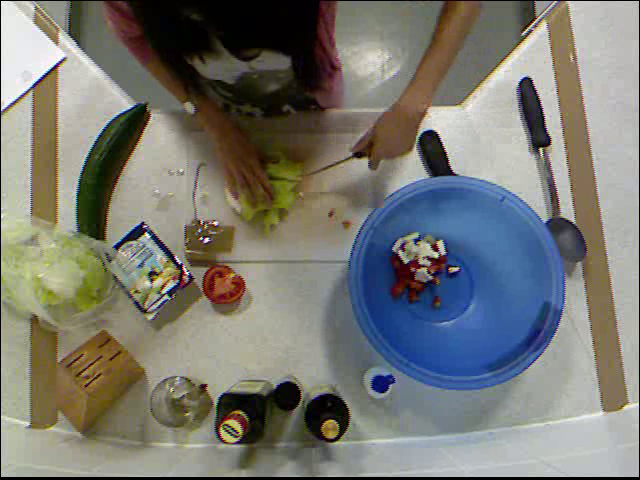}
		\caption{frame at time $t$.}
		\label{sfig:compare_optical_flow_f2}
	\end{subfigure}
	\\
	\begin{subfigure}[t]{0.45\linewidth}
		\includegraphics[width=\linewidth]{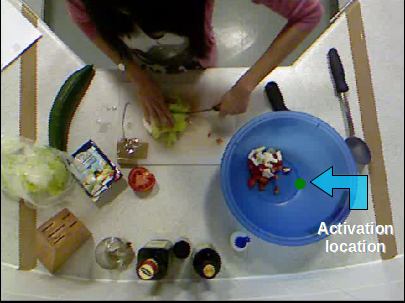}
		\caption{no motion vectors found on the background region.}
		\label{sfig:compare_optical_flow_background}
	\end{subfigure}
	\begin{subfigure}[t]{0.45\linewidth}	
		\includegraphics[width=\linewidth]{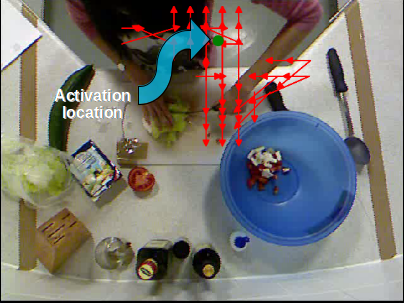}
		\caption{motion vectors found on the moving region.}
		\label{sfig:compare_optical_flow_moving}
	\end{subfigure}
	\\
	\begin{subfigure}[t]{0.45\linewidth}
		\frame{\includegraphics[width=\linewidth]{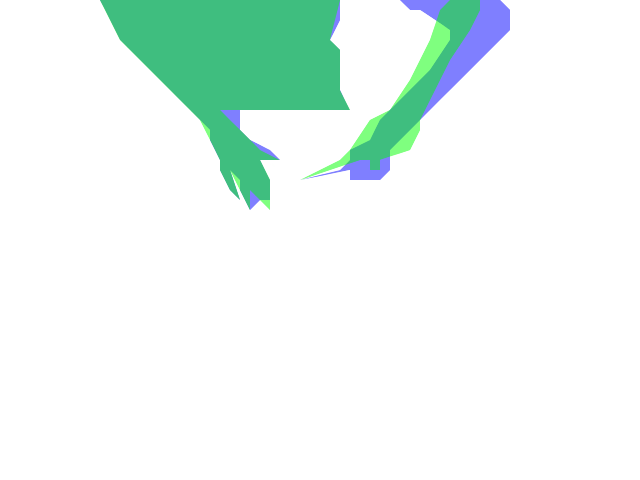}}
		\caption{the person at time $t$-1 (blue) and $t$ (green).}
		\label{sfig:compare_optical_flow_move}
	\end{subfigure}
	\begin{subfigure}[t]{0.45\linewidth}
		\includegraphics[width=\linewidth]{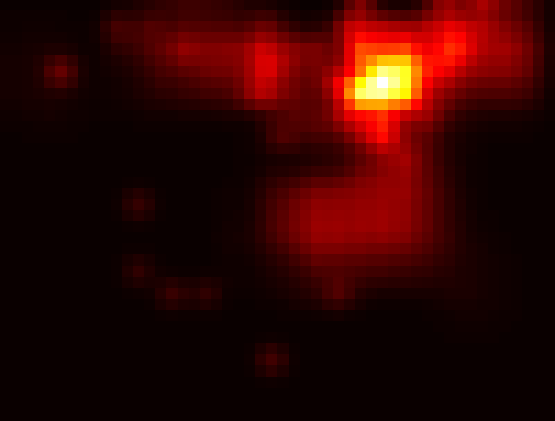}
		\caption{visualization of motion in feature space.}
		\label{sfig:compare_optical_flow_energy}
	\end{subfigure}
    \vspace{-2mm}
	\caption{Visualization of difference of adaptive receptive fields for action \textit{cutting lettuce} in 50 Salads dataset: (a) and (b) are two consecutive frames; (c) and (d) are motion vectors at background and moving regions (green dots indicate activation locations and red arrows indicate motion vectors); (e) is the manually defined mask of the person at time $t-1$ and $t$; and (f) is the energy of motion field in feature space, computed by aggregating motion vectors in all deformable convolution layers.}
	\label{fig:compare_optical_flow}
\end{figure}

\textit{Action detection}, \textit{a.k.a} action segmentation, addresses the task of classifying every frame of a given video, containing multiple action segments, as one out of a fixed number of defined categories, including a category for unknown actions. This is contrary to the simpler task of \textit{action recognition}, wherein a given video is pre-segmented and guaranteed to be one of the provided action classes \cite{kang2016review}. 

Fine-grained actions are a special class of actions which can only be differentiated by subtle differences in motion patterns. Such actions are characterized by high inter-class similarity \cite{rohrbach2012mpii, singh2016merl}, \ie it is difficult, even for humans, to distinguish two different actions just from observing individual frames. Unlike generic action detection, which can largely rely on ``what'' is in a video frame to perform detection, fine-grained action detection requires additional reason about ``how'' the objects move across several video frames.
In this work, we consider the fine-grained action detection setting. 

The pipeline of fine-grained action detection generally consists of two steps: (1) spatio-temporal feature extraction and (2) long-temporal modeling. The \textbf{first step} models spatial and short-term temporal information by looking at a few consecutive frames. Traditional approaches tackle this problem by decoupling spatial and temporal information in different feature extractors and then combining the two streams with a fusion module. Optical flow is commonly used for such short-term temporal modeling \cite{feichtenhofer2016spatialtemporal, feichtenhofer2016twostream, simonyan14twostream, singh2016merl, singh2016egonet}. However, optical flow is usually computationally expensive and may suffer from noise introduced by data compression \cite{lea2017tempconvnet, lea2016segmental}. Other approaches use Improved Dense Trajectory (IDT) or Motion History Image (MHI) as an alternative to optical flow \cite{bobick2001representation, lea2016segmental, wang2013idt}. Recently, there have been efforts to model motion in video using variants of 3D convolutions \cite{carreira2017i3d, kang2016tubelets, tran2015c3d}. In such cases, motion modeling is somewhat limited by receptive fields of standard convolutional filters \cite{holschneider1990atrous, yu2016dilated, yu2017dilatednet}.

The \textbf{second step} models long-term dependency of extracted spatio-temporal features over the whole video, \eg bi-directional LSTM \cite{singh2016merl}, spatial-temporal CNN (ST-CNN) with segmentation models \cite{lea2016segmental}, temporal convolutional networks (TCN) \cite{lea2017tempconvnet}, and temporal deformable residual networks (TDRN) \cite{lei2018temporal}. Recent works that focused on modeling long-term dependency have usually relied on existing features \cite{lea2017tempconvnet, lea2016segmental, lei2018temporal}. In this work, we create efficient short-term spatio-temporal features which are very effective in modeling fine-grained motion.

Instead of modeling temporal information with optical flow, we learn temporal information in the \textit{feature space}. This is accomplished by utilizing our proposed \textit{locally-consistent deformable convolution (LCDC)}, which is an extension of the standard deformable convolution \cite{dai17dcn}. At a high-level, we model motion by evaluating the local movements in adaptive receptive fields over time (as illustrated in \figref{fig:compare_optical_flow}). Adaptive receptive fields can focus on important parts \cite{dai17dcn} in a frame, thus using them helps focus on movements of interesting regions. On the other hand, traditional optical flow tracks all possible motion, some of which may not be necessary. Furthermore, we enforce a local coherency constraint over the adaptive receptive fields to achieve temporal consistency.

To demonstrate the effectiveness of our approach, we evaluate on two standard fine-grained action detection datasets: 50 Salads \cite{stein2013salads} and Georgia Tech Egocentric Activities (GTEA) \cite{fathi2011gtea}. We also show that our features, without any optical flow guidance, are robust and outperform features from original networks. Additionally, we perform quantitative evaluation of the learned motion using ablation studies to demonstrate the power of our model in capturing temporal information.

Our main contributions are:
\textit{(1) Modeling motion in feature space} using changes in adaptive receptive fields over time, instead of relying on pixel space as in traditional optical flow based methods. To the best of our knowledge, we are the first to extract temporal information from receptive fields. 
\textit{(2) Introducing local coherency constraint} to enforce consistency in motion. The constraint reduces redundant model parameters, making motion modeling more robust. 
\textit{(3) Constructing a backbone single-stream network to jointly learn spatio-temporal features}. This backbone network is flexible and can be used in consonance with other long-temporal models. Furthermore, we prove that the network is capable of representing temporal information 
with a behavior equivalent to optical flow.
\textit{(4) Significant reduction of model complexity} is achieved without sacrificing performance by using local coherency constraint. This reduction is proportional to the number of deformable convolution layers. Our single-stream approach is computationally more efficient than traditional two-stream networks, as they require expensive optical flow and multi-stream inference.
\section{Related work}
\label{sec:related_work}

An extensive body of literature exists for features, temporal modeling, and network architectures within the context of action detection. In this section, we will review the most recent and relevant papers related to our approach.

\noindent\textbf{Spatio-temporal features.} Spatio-temporal features are crucial in the field of video analysis. Usually, the features consist of spatial cues (extracted from RGB frames) and temporal cues over a \textit{short} period of time. Optical flow \cite{lucas1981opticalflow} is often used to model temporal information. However, it was found to suffer from noise due to video compression and insufficient to capture small motion \cite{lea2017tempconvnet, lea2016segmental}. It is also generally computationally expensive. Other solutions to model temporal information include Motion History Image (MHI) \cite{bobick2001representation}, leveraging the difference of multiple consecutive frames, and Improved Dense Trajectory (IDT) \cite{wang2013idt}, combining HOG \cite{dalal2005hog}, HOF \cite{wang2013idt}, and Motion Boundary Histograms (MBH) descriptors \cite{dalal2006mbh}.

To combine spatial and (short) temporal components, Lea \etal \cite{lea2016segmental} stacked an RGB frame with MHI as input to a VGG-like network to produce features (which they refereed to as SpatialCNN features). Simonyan and Zisserman \cite{simonyan14twostream} proposed a two-stream network, combining scores from separate appearance (RGB) and motion streams (stacked optical flows). The original approach was improved by more advanced fusion in \cite{feichtenhofer2016spatialtemporal, feichtenhofer2016twostream}. A different school of thought models motion using variants of 3D convolutions including C3D proposed in \cite{tran2015c3d}. Inflated 3D (I3D) network, leveraging 3D convolutions within a two-stream setup was proposed in \cite{carreira2017i3d}. To cope with egocentric motion captured by head-mounted cameras, Singh \etal introduced a third stream (EgoStream) in \cite{singh2016egonet}, capturing the relation of hands, head, and eyes motion. \cite{singh2016merl} further used four streams (two appearance and two motion streams) in Multi-Stream Network (MSN). Each domain (spatial and temporal) has a global view (whole frame) and a local view (cropped by motion tracker).

\noindent\textbf{Long-temporal modeling.} While spatio-temporal features are usually extracted over short periods of time, some form of long-temporal modeling is performed to capture long-term dependencies within the entirety of a video containing an action sequence. In \cite{lea2017tempconvnet} Spatio-temporal CNN (ST-CNN) was introduced to combine SpatialCNN features using a 1D convolution that spans over a long period of time. Singh \etal learned the long-term dependency from MSN features (four-stream) using bi-directional LSTMs \cite{singh2016merl}. More recently, \cite{lea2017tempconvnet} proposed two Temporal Convolution Networks (TCN): DilatedTCN and Encoder-Decoder TCN (ED-TCN). These networks fused SpatialCNN features and captured long-temporal patterns by convolving them in the time-domain. A Temporal Deformable Residual Networks (TDRN) was proposed in \cite{lei2018temporal} to model long-temporal information by applying a deformable convolution in the time domain. The TCN model was also further improved with multi stage mechanism in Multi-Stage TCN (MS-TCN) \cite{farha2019mstcn}.

\noindent\textbf{Network architectures.} Pre-trained architectures for image classification, such as VGG, Inception, ResNet \cite{he2016resnet, simonyan2014vgg, szegedy2015inception} are the most important determinants of the performance of the main down-stream vision tasks. Many papers have focused on improving the recognition accuracy by innovating on the network architecture. 
In standard convolutions, the convolutional response always comes from a local region. Dilated convolutions have been introduced to overcome this problem by changing the shape of receptive fields with some dilation patterns \cite{holschneider1990atrous, yu2016dilated, yu2017dilatednet}. In 2017, Dai \etal. \cite{dai17dcn} introduced deformable convolutional networks with adaptive receptive fields. The method is more flexible since the receptive fields depend on input and can approximate an arbitrary object's shape. We leverage on the advances of \cite{dai17dcn}, specifically the adaptive receptive fields from the model to capture motion in the \textit{feature space}. We further add a local coherency constraint on receptive fields in order to ensure that the motion fields are consistent. This constraint also plays a major role in reducing model complexity.
\section{Locally-Consistent Deformable Convolution Networks}
\label{sec:method}

\begin{figure}
	\centering
	\includegraphics[width=.9\linewidth]{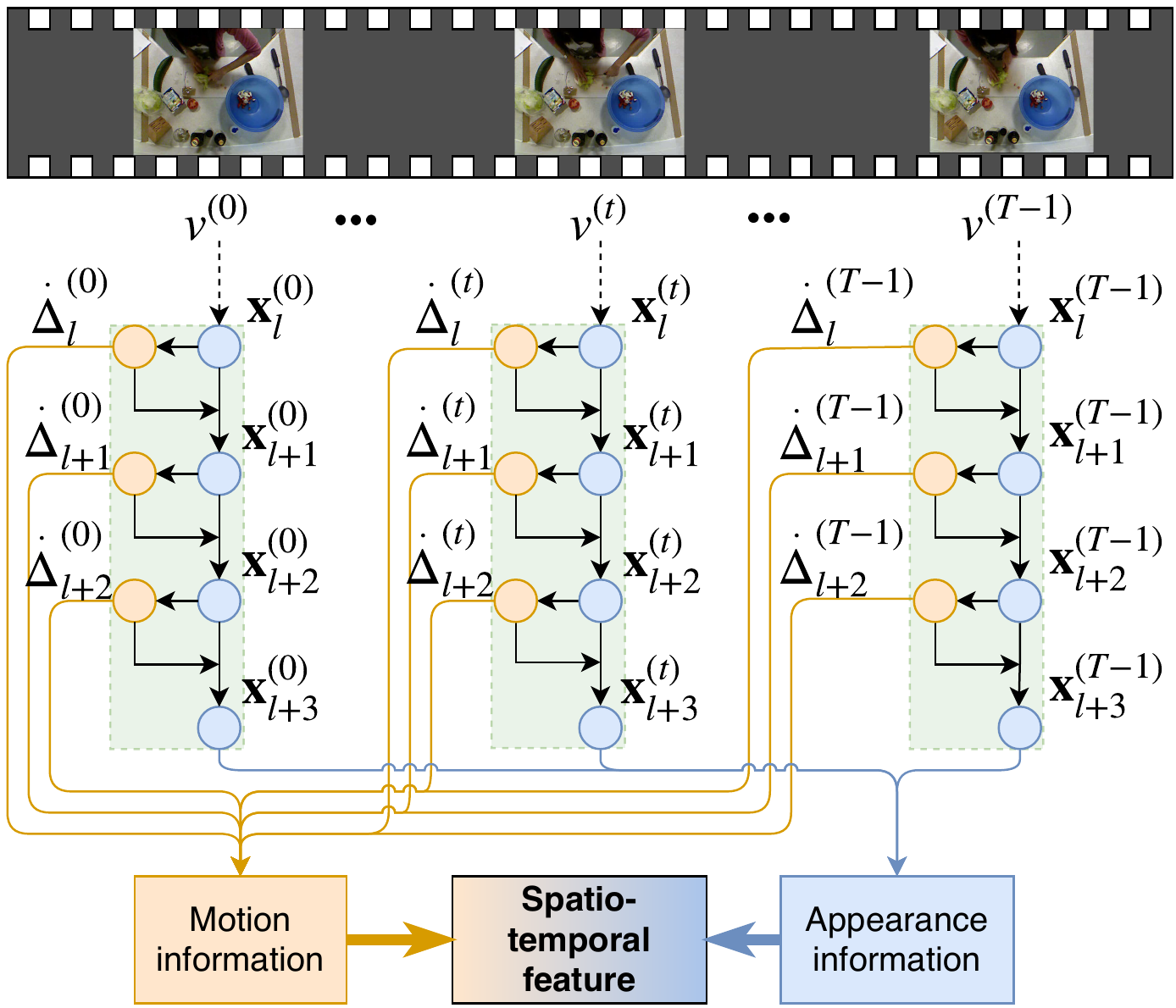}
    \vspace{-2mm}
	\caption{Network architecture of the proposed LCDC across multiple frames $v^{(t)}$. Appearance information comes from the last layer while motion information is extracted directly from deformation $\dot{\Delta}$ in the feature space instead of from a separate optical flow stream. Weights are shared across frames over time.}
	\label{fig:archi_overview}
\end{figure}

Our architecture builds upon deformable convolutional networks with an underlying ResNet CNN. While a deformable convolutional network has been shown to succeed in the task of object detection and semantic segmentation, it is not directly designed for fine-grained action detection. However, we observe that deformable convolution layers have a byproduct, the \textit{adaptive receptive field}, which can capture motion very naturally.

At a high level, an adaptive receptive field in a deformable convolution layer can be viewed as an aggregation of important pixels, as the network has the flexibility to change where each convolution samples from. In a way, the adaptive receptive fields are performing some form of key-points detection. Therefore, our hypothesis is that, if the key-points are consistent across frames, we can model motion by taking the difference in the adaptive receptive fields across time. As a deformable convolution can be trained end-to-end, our network can learn to model motion at hidden layers of the network. Combining this with spatial features leads to a powerful spatio-temporal feature.

We illustrate the intuition of our method in \figref{fig:compare_optical_flow}. The motion here is computed using difference in adaptive receptive fields on multiple \textit{feature spaces} instead of pixel space as in optical flow. Two consecutive frames of action \textit{cutting lettuce} from 50 Salads dataset are shown in \figref{sfig:compare_optical_flow_f1} and \figref{sfig:compare_optical_flow_f2}. \figref{sfig:compare_optical_flow_move} shows masks of the person to illustrate how the action takes place. We also show the motion vectors corresponding to different regions in \figref{sfig:compare_optical_flow_background} and \figref{sfig:compare_optical_flow_moving}. Red arrows are used to describe the motion and green dots are used to show the corresponding activation units. We suppress motion vectors with low values for the sake of visualization. In \figref{sfig:compare_optical_flow_background}, the activation unit lies on a background region (cut ingredients inside the bowl) and so there is no motion recorded as the difference between two adaptive receptive fields of background region over time is minimal. However, we can find motion in \figref{sfig:compare_optical_flow_moving} (the field of red arrows) because the activation unit lies on a moving region, \ie the arm region. The motion field at all activation units is seen in \figref{sfig:compare_optical_flow_energy}, where the field's energy corresponds to the length of motion vectors at each location. The motion field is excited around the moving region (the arm) while suppressed in the background. Therefore, this highly suggests that the motion information we extract can be used as an alternative solution to optical flow. A schematic of the proposed network architecture is shown in \figref{fig:archi_overview}.

\begin{figure}
    \centering
    \includegraphics[width=\linewidth]{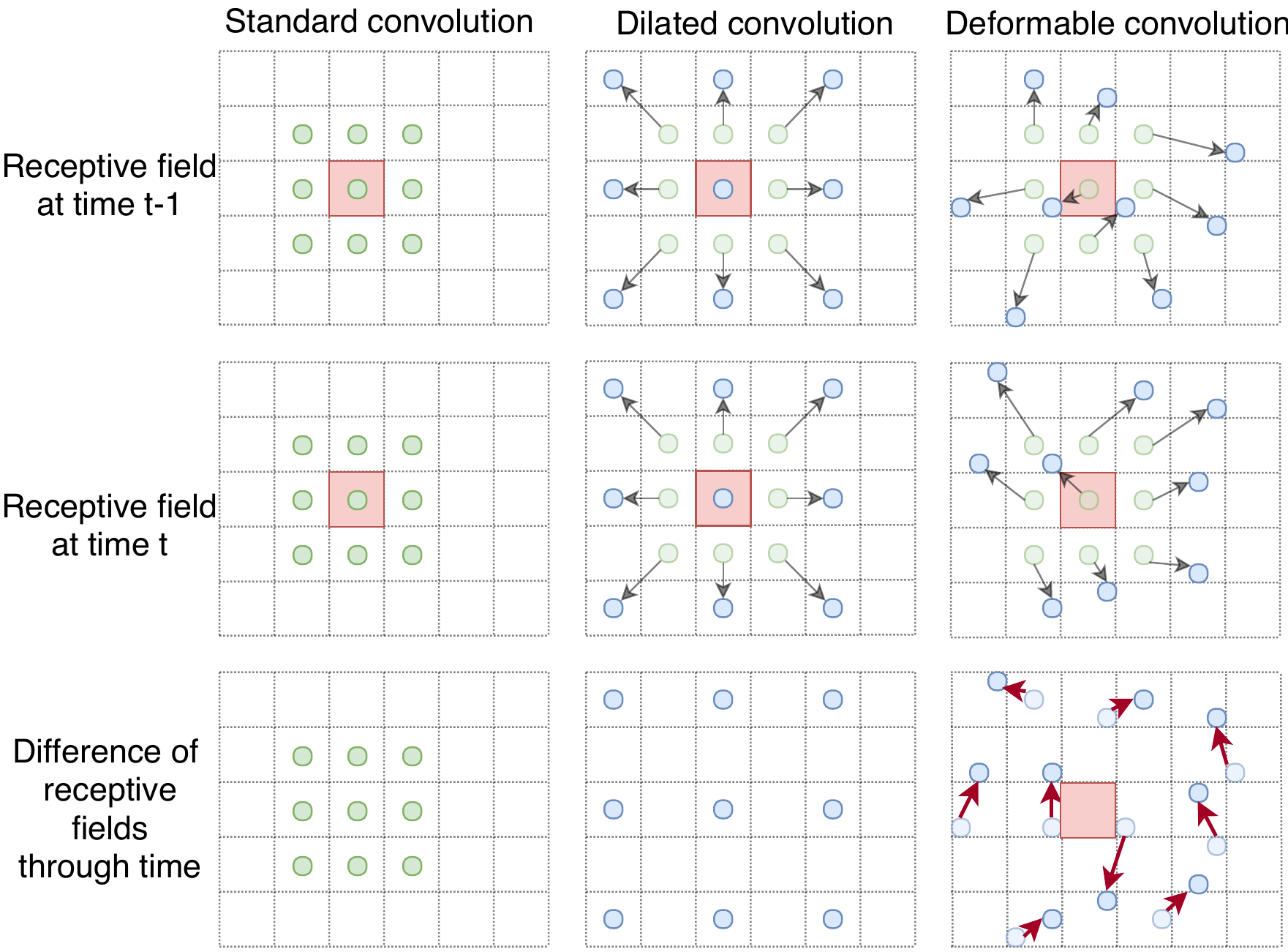}
    \vspace{-2mm}
    \caption{Illustration of temporal information modeled by the difference of receptive fields at a single location in 2D. Only deformable convolution can capture temporal information (shown with red arrows). Related to \equref{equ:deconv} and \equref{equ:diff_receptive_field}, $n$ is red square, $n$+$k$ are green dots, $\ddot{\Delta}_{n,k}$ are black arrows, $n$+$k$+$\ddot{\Delta}_{n,k}$ are blue dots, and $\ddot{\br}$ are red arrows.}
    \label{fig:receptive_fields}
\end{figure}

\subsection{Deformable convolution}
\label{ssec:dconv}

We first briefly review the deformable convolution layers, before going into a concrete description of the construction of the network architecture. Let $\bx$ be the input signal such that $\bx \in \R^N$. The standard convolution is defined as:
\begin{equation}
    \by[n] = \sum_k \bw[-k] \bx \left[ n+k \right],
\end{equation}
where $\bw \in \R^{K}$ is the convolutional kernel, $n$ and $k$ are the signal and kernel indices ($n$ and $k$ can be treated as multidimensional indices). The deformable convolution proposed in \cite{dai17dcn} is thus defined as:
\begin{equation}
    \by[n] = \sum_k \bw[-k] \bx \left( n + k + \ddot{\Delta}_{n, k} \right),
    \label{equ:deconv}
\end{equation}
where $\ddot{\Delta} \in \R^{N \times K}$ represents the deformation offsets of deformable convolution. These offsets are learned from another convolution with $\bx$ \ie $\ddot{\Delta}_{n,k} = (\bh_k * \bx)[n]$, where $\bh$ is a different kernel. Note that we use parentheses $(\cdot)$ instead of brackets $[\cdot]$ for $\bx$ in \equref{equ:deconv} because the index $n + k + \ddot{\Delta}_{n,k}$ requires interpolation as $\ddot{\Delta}$ is fractional.

\subsection{Modeling temporal information with adaptive receptive fields}
\label{ssec:adapt_recept}
We define the adaptive receptive field of a deformable convolution at time $t$ as $\ddot{\mathbf{F}}^{(t)} \in \R^{N \times K}$ where $\ddot{\mathbf{F}}^{(t)}_{n, k}= n + k + \ddot{\Delta}_{n,k}^{(t)}$. To extract motion information from adaptive receptive fields, we take the difference of the receptive fields through time, which we denote as:
\begin{equation}
	\ddot{\br}^{(t)} = \ddot{\mathbf{F}}^{(t)} - \ddot{\mathbf{F}}^{(t-1)} = \ddot{\Delta}^{(t)} - \ddot{\Delta}^{(t-1)}.
	\label{equ:diff_receptive_field}
\end{equation}
It can be seen that the locations $n+k$ are canceled, going from $t-1$ to $t$ in \equref{equ:diff_receptive_field}, leaving only the difference of deformation offsets. Given $T$ input feature maps with spatial dimension $H \times W$, we can construct $T$ different $\ddot{\Delta}^{(t)} \rvert_{t=0}^{T-1}$, resulting in $T-1$ motion fields $\ddot{\br}^{(t)} \rvert_{t=0}^{T-2}$ with the same spatial dimension. Therefore, we can model different motion at different positions $n$ and time $t$.

\figref{fig:receptive_fields} further illustrates the meaning of $\ddot{\br}^{(t)}$ in 2D for different types of convolutions. Red square shows the current activation location, green dots show the standard receptive fields, and blue dots show the receptive fields after adding deformation offsets. In the last row, red arrows show the changes of receptive field from time $t-1$ (faded blue dots) to time $t$ (solid blue dots). Readers should note that there are no red arrows for standard convolution and dilated convolution because the offsets are either zero or identical. Red arrows only appear in deformable convolution, which motivates modeling of temporal information.

\begin{figure}
	\centering
	\includegraphics[width=\linewidth]{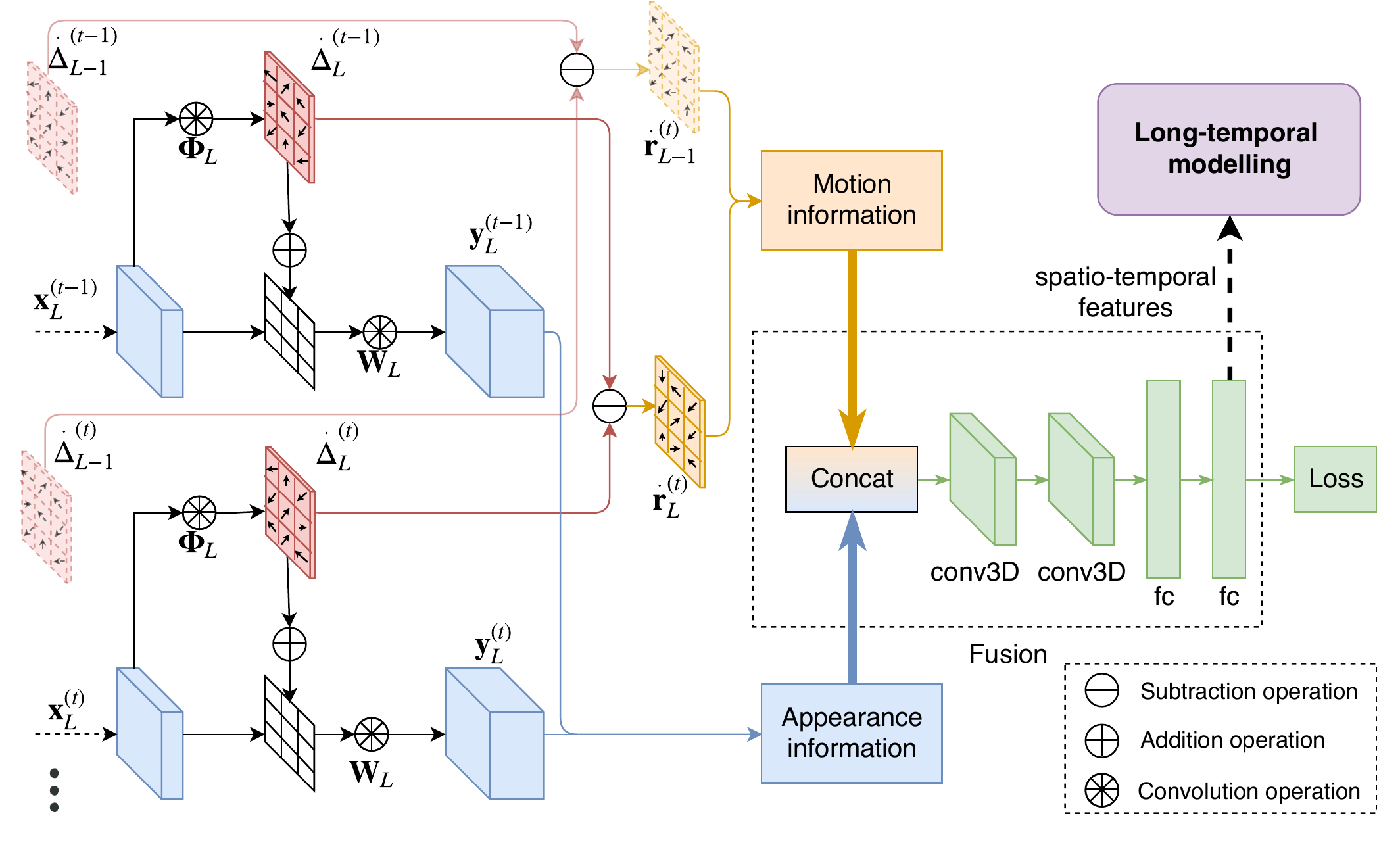}
    \vspace{-2mm}
	\caption{A more detailed view of our network architecture with the fusion module. Appearance information comes from output of the last layer while motion information comes from aggregating $\dot{\br}$ from multiple layers. Outputs of the final fc layer can be flexibly used as the features for any long-temporal modeling networks.}
	\label{fig:mrDeRF_archi}
\end{figure}

\subsection{Locally-consistent deformable convolution}
\label{ssec:dconv_motion_constraint}
Directly modeling motion using $\ddot{\br}$ is not very effective because there is no guarantee of local consistency in receptive fields in the original deformable convolution formulation. This is because $\ddot{\Delta}_{n,k}$ is defined on both location ($n$) and kernel ($k$) indices, which essentially corresponds to $\bx[m]$, where $m = n + k$. However, there are multiple ways to decompose $m$, \ie $m = n+k = (n-l) + (k+l)$, for any $l$. Therefore, one single $\bx[m]$ is deformed by multiple $\ddot{\Delta}_{n-l, k+l}$, with different $l$. This produces inconsistency when we model $\ddot{\br}^{(t)}$ in \equref{equ:diff_receptive_field}, as there can be multiple motion vectors corresponding to the same location. While local consistency could be learned as a side-effect of the training process, it is still not explicitly formulated in the original deformable convolution formulation.

In order to enforce consistency, we propose a locally-consistent deformable convolution (LCDC):
\begin{equation}
    \by[n] = \sum_k \bw[-k] \bx \left( n + k + \dot{\Delta}_{n+k} \right),
    \label{equ:cdc}
\end{equation}
for $\dot{\Delta} \in \R^N$. LCDC is a special case of deformable convolution where
\begin{equation}
    \ddot{\Delta}_{n,k} = \dot{\Delta}_{n+k}, \quad \forall n, k.
    \label{equ:cdc_constraint}
\end{equation}
We name this as \textit{local coherency constraint}. The interpretation of LCDC is that instead of deforming the receptive field as in \equref{equ:deconv}, we can deform the input signal instead. Specifically, LCDC in \equref{equ:cdc} can be rewritten as:
\begin{equation}
    \by[n] = \sum_k \bw[-k] \tilde{\bx}[n+k] = (\tilde{\bx}*\bw)[n],
\end{equation}
where
\begin{equation}
    \tilde{\bx}[n] = (D_{\dot{\Delta}} \{\bx\})[n] = \bx\left( n + \dot{\Delta}_n \right)
\end{equation}
is a deformed version of $\bx$ and $*$ is the standard convolution ($D_{\dot{\Delta}} \{\cdot\}$ is defined as the deforming operation by offset $\dot{\Delta}$).

Both $\ddot{\Delta}$ and $\dot{\Delta}$ are learned via a convolution layer. Recall that $\ddot{\Delta}_{n,k} = (\bh_k * \bx)[n]$, where $\bx \in \R^N$ and $\ddot{\Delta} \in \R^{N \times K}$. $\dot{\Delta}$ is constructed similarly, \ie 
\begin{equation}
    \dot{\Delta}_n = (\Phi * \bx)[n],
    \label{equ:dot_delta}
\end{equation}
where $\dot{\Delta} \in \R^N$. Since $\ddot{\Delta}$ and $\dot{\Delta}$ share the same spatial dimension $N$ and they can be applied for different time frames, $\dot{\Delta}$ can also model motion at different positions and times.

Furthermore, $\dot{\Delta}$ only needs a kernel $\Phi$, while $\ddot{\Delta}$ requires multiple $\bh_k$. Therefore, LCDC is more memory-efficient as we can reduce memory cost $K$ times. Implementation-wise, given input feature map $\bx \in \R^{H \times W \times C}$, then $\ddot{\Delta} \in \R^{(H \times W) \times (G \times K_h \times K_w \times 2)}$, where $G$ is the number of deformable groups, $K_h$ and $K_w$ are the height and width of kernels, and 2 indicates that offsets are 2D vectors. However, the dimensionality of LCDC offsets $\dot{\Delta}$ is only $\R^{H \times W \times 2}$. We also drop the number of deformable groups $G$ since we want to model one single type of motion between two time frames. Therefore, the reduction in this case is $G \times K_h \times K_w$ times. The parameter reduction is proportional to the number of deformable convolution layers that are used.

We now show that LCDC can effectively model both appearance and motion information in a single network, as the difference $\dot{\br}^{(t)} = \dot{\Delta}^{(t)} - \dot{\Delta}^{(t-1)}$ has a behavior equivalent to motion information produced by optical flow.
\begin{proposition}
Suppose that two inputs $\bx^{(t-1)}$ and $\bx^{(t)}$ are related through a motion field, \ie
\begin{equation}
    \bx^{(t)}(s) = \bx^{(t-1)}\left( s - o(s) \right),
    \label{equ:optical_flow_relation}
\end{equation}
where $o(s)$ is the motion at location $s \in \R^2$, and $\bx^{(t)}$ is assumed to be locally varying. Then the corresponding LCDC outputs with $\bw \neq 0$:
\begin{align*}
    \by^{(t)} &= (D_{\dot{\Delta}^{(t)}} \{\bx^{(t)}\})*\bw, \\
    \by^{(t-1)} &= (D_{\dot{\Delta}^{(t-1)}} \{\bx^{(t-1)}\})*\bw
\end{align*}
are consistent, \textit{i.e.} $\by^{(t-1)} = \by^{(t)}$, if and only if $\;\forall n$,
\begin{equation}
    \dot{\br}^{(t)}_n = \dot{\Delta}_n^{(t)} - \dot{\Delta}_n^{(t-1)} = o\left(n + \dot{\Delta}_n^{(t)}\right).
    \label{equ:encoded_motion}
\end{equation}
Notice that in pixel space, $\bx$ are input images and $o(s)$ is the optical flow at $s$. In latent space, $\bx$ are intermediate feature maps and $o(s)$ is the motion of feature.
\end{proposition}

\begin{proof}
With the connection of LCDC to standard convolution, under the assumption that $\bw \neq 0$, we have:
\begin{align*}
    \by^{(t)} &= \by^{(t-1)} \\
    \Leftrightarrow D_{\dot{\Delta}^{(t)}} \{\bx^{(t)}\} &= D_{\dot{\Delta}^{(t-1)}} \{\bx^{(t-1)}\} \\
    \Leftrightarrow \bx^{(t)} \left( n + \dot{\Delta}_n^{(t)} \right) &= \bx^{(t-1)} \left(n + \dot{\Delta}_n^{(t-1)} \right), \forall n.
\end{align*}
Substituting the LHS in the motion relation in \equref{equ:optical_flow_relation}, we obtain the following equivalent conditions $\forall n$:
\begin{align*}
    \bx^{(t-1)}\left( n + \dot{\Delta}_n^{(t)} - o(n + \dot{\Delta}_n^{(t)}) \right) &= \bx^{(t-1)} \left(n + \dot{\Delta}_n^{(t-1)} \right) \\
    \Leftrightarrow \dot{\Delta}_n^{(t)} - o(n + \dot{\Delta}_n^{(t)}) &= \dot{\Delta}_n^{(t-1)} \\
    \Leftrightarrow o\left(n + \dot{\Delta}_n^{(t)}\right) &= \dot{\Delta}_n^{(t)} - \dot{\Delta}_n^{(t-1)} = \dot{\br}^{(t)}_n.
\end{align*}
(since $\bx^{(t)}$ is locally varying).
\end{proof}

The above result shows that by enforcing consistent output and sharing weights $\bw$ across frames, the learned deformed map $\Dot{\Delta}_n^{(t)}$ encodes motion information, as in \equref{equ:encoded_motion}. Hence, we can effectively model both appearance and motion information in a single network with LCDC, instead of using two different streams.

\subsection{Spatio-temporal features}
\label{ssec:spatiotemporal}
To create the spatio-temporal feature, we further concatenate across channel dimensions the learned motion information $\dot{\br}^{(t)}$ from multiple layers with appearance features (output of the last layer $\mathbf{y}_{L}^{(t)}$). We illustrate this process in \figref{fig:mrDeRF_archi}. To model the fusion mechanism, we used two 3D convolutions followed by two fc layers. Each 3D convolution unit was followed by batch normalization, ReLU activation, and 3D max pooling to gradually reduce temporal dimension (while the spatial dimension is retained). Outputs of the final fc layer can be flexibly used as the features for any long-temporal modeling networks, such as ST-CNN~\cite{lea2016segmental}, Dilated-TCN~\cite{lea2017tempconvnet}, or ED-TCN~\cite{lea2017tempconvnet}.
\section{Experiments}
\label{sec:experiment}

\subsection{Implementation details}
\label{ssec:imp_detail}
We implemented our approach using ResNet50 with deformable convolutions as backbone (at layers \texttt{conv5a}, \texttt{conv5b}, and \texttt{conv5c} as in \cite{dai17dcn}). Local coherency constraints were added on all existing deformable convolutions layers. For the fusion module, we used
a spatial kernel with size 3 and stride 1; and temporal kernel with size 4 and stride 2. We also used pooling with size 2 and stride 2 in 3D max pooling. Temporal dimension was collapsed by averaging. The network ended with two fully connected layers. Standard cross-entropy loss with weight regularization was used to optimize the model. After training, LCDC features (last fc layer) were extracted and incorporated into long-temporal models.
All data cross-validation splits followed the settings of \cite{lea2017tempconvnet}. Frames were resized to 224x224 and augmented using random cropping and mean removal. Each video snippet contained 16 frames after sampling. For training, we downsampled to 6fps on 50 salads and 15 fps on GTEA, because of different motion speeds, to make sure one video snippet contained enough information to describe motion. For testing, features were downsampled with the same frame rates as other papers for comparison. We used the common Momentum optimizer \cite{qian1999momentum} (with momentum of 0.9) and followed the standard procedure of hyper-parameter search. Each training routine consisted of 30 epochs; learning rate was initialized as $10^{-4}$ and decayed every $10$ epochs with a decaying rate of $0.96$.

\subsection{Datasets}
\label{ssec:dataset}
We evaluate our approach on two standard datasets, namely, 50 Salads dataset and GTEA dataset.

\noindent \textbf{50 Salads Dataset \cite{stein2013salads}:} This dataset contains 50 salad making videos from multiple sensors. We only used RGB videos in our work. Each video lasts from 5-10 minutes, containing multiple action instances. We report results for \textit{mid} (17 action classes) and \textit{eval} granularity level (9 action classes) to be consistent with results reported in \cite{lea2017tempconvnet, lea2016segmental,  lei2018temporal}.

\noindent \textbf{Georgia Tech Egocentric Activities (GTEA) \cite{fathi2011gtea}:} This dataset contains 28 videos of 7 action classes, performed by 4 subjects. The camera in this dataset is head-mounted, thus introducing more motion instability. Each video is about 1 minute long and has around 19 different actions on average.

\subsection{Baselines}
\label{ssec:baselines}
We compare LCDC with several baselines including (1) methods which do not involve long-temporal modeling where comparison is at spatio-temporal feature level (SpatialCNN) and (2) methods with long-temporal modeling (ST-CNN, DilatedTCN, and ED-TCN).

\noindent \textbf{SpatialCNN \cite{lea2016segmental}:} a VGG-like model that learns both spatial and \textit{short-term} temporal information by stacking an RGB frame with the corresponding MHI (the difference between frames over a \textit{short} period of time). MHI is used for both 50 Salads and GTEA datasets instead of optical flow as optical flow was observed to suffer from small motion and data compression noise \cite{lea2017tempconvnet, lea2016segmental}. SpatialCNN features are also used as \textit{inputs} for ST-CNN, DilatedTCN, ED-TCN, and TDRN.

\noindent \textbf{ST-CNN \cite{lea2016segmental}, DilatedTCN \cite{lea2017tempconvnet}, and ED-TCN \cite{lea2017tempconvnet}:} are long-temporal modeling frameworks. Long-term dependency was modeled using a 1D convolution layer in ST-CNN, stacked dilated convolutions in DilatedTCN, and an encoder-decoder with pooling and up-sampling in ED-TCN. All three frameworks were originally proposed with SpatialCNN features as their input. We incorporated LCDC features into these long-temporal models and compared with the original results.

We obtained the publicly available implementations of  ST-CNN, DilatedTCN, and ED-TCN from \cite{leagit}. On incorporating LCDC features into these models, we observed that training from scratch can become sensitive to random initialization. This is likely because these long-temporal models have a low complexity (\ie only a few layers) and the input features are not augmented. We ran each long-temporal model (with LCDC features) five times and report means and standard deviations over multiple metrics. For completeness, we have also included original results from TDRN (where the input was SpatialCNN features as well) \cite{lei2018temporal}. However, TDRN's implementation was not publicly available so we were unable to incorporate LCDC with TDRN.

\subsection{Results}
\label{ssec:results}
\begin{table*}
    \footnotesize
    \centering
    \begin{tabularx}{\textwidth}{llccccccc}
        \thickhline
         & Model & Spatial comp & Temporal comp (short) & Long-temporal & F1@10 & Edit & Acc  \\
        \hline \hline
        \multirow{9}{*}{\rotatebox[origin=c]{90}{Mid}} 
         & SpatialCNN \cite{lea2016segmental} & RGB & MHI & - & 32.3 & 24.8 & 54.9 \\
         & (SpatialCNN) + ST-CNN \cite{lea2016segmental}       & RGB & MHI & 1D-Conv    & 55.9 & 45.9 & 59.4 \\
         & (SpatialCNN) + DilatedTCN \cite{lea2017tempconvnet} & RGB & MHI & DilatedTCN & 52.2 & 43.1 & 59.3 \\
         & (SpatialCNN) + ED-TCN \cite{lea2017tempconvnet}     & RGB & MHI & ED-TCN     & 68.0 & 59.8 & 64.7 \\
         & (SpatialCNN) + TDRN \cite{lei2018temporal}          & RGB & MHI & TDRN       & (72.9) & (66.0) & (68.1) \\
         \cline{2-8}
         & LCDC            & RGB & Learned deformation & -       & 43.99 & 33.38 & 67.27 \\
         & LCDC + ST-CNN     & RGB & Learned deformation & 1D-Conv & 60.01$\pm$0.42 & 51.35$\pm$0.12 & 68.45$\pm$0.15 \\
         & LCDC + DilatedTCN & RGB & Learned deformation & DilatedTCN & 58.21$\pm$0.59 & 48.54$\pm$0.52 & 69.28$\pm$0.25 \\
         & LCDC + ED-TCN     & RGB & Learned deformation & ED-TCN  & \textbf{73.75}$\pm$\textbf{0.54} & \textbf{66.94}$\pm$\textbf{1.33} & \textbf{72.12}$\pm$\textbf{0.41} \\
        \hline \hline
        \multirow{8}{*}{\rotatebox[origin=c]{90}{Eval}}
         & Spatial CNN \cite{lea2016segmental}  & RGB & MHI & - & 35.0 & 25.5 & 68.0 \\
         & (SpatialCNN) + ST-CNN \cite{lea2016segmental}       & RGB & MHI & 1D-Conv    & 61.7 & 52.8 & 71.3 \\
         & (SpatialCNN) + DilatedTCN \cite{lea2017tempconvnet} & RGB & MHI & DilatedTCN & 55.8 & 46.9 & 71.1 \\
         & (SpatialCNN) + ED-TCN \cite{lea2017tempconvnet}     & RGB & MHI & ED-TCN & 76.5 & 72.2 & 73.4 \\
         \cline{2-8}
         & LCDC        & RGB & Learned deformation & -       & 56.56 & 45.77 & 77.59 \\
         & LCDC + ST-CNN & RGB & Learned deformation & 1D-Conv & 70.46$\pm$0.41 & 62.71$\pm$0.46 & 77.84$\pm$0.26 \\
         & LCDC + DilatedTCN & RGB & Learned deformation & DilatedTCN & 67.59$\pm$0.42 & 58.97$\pm$0.55 & 78.29$\pm$0.29 \\
         & LCDC + ED-TCN & RGB & Learned deformation & ED-TCN  & \textbf{80.22}$\pm$\textbf{0.21} & \textbf{74.56}$\pm$\textbf{0.70} & \textbf{78.90}$\pm$\textbf{0.25} \\
        \thickhline
    \end{tabularx}
    \vspace{-2mm}
    \caption{Results on 50 salads dataset (\textit{mid} and \textit{eval}-level). Learned deformation is $\dot{\Delta}$ in \equref{equ:dot_delta}. Means and standard deviations over five runs are reported for LCDC with long-temporal models. Results of baselines are directly reported from their original publications. Please note that since TDRN implementation was not publicly available, LCDC features were not incorporated into TDRN and hence the TDRN results (in parentheses) are not directly comparable with LCDC results.
    }
    \label{tab:results_50_salads}
\end{table*}

\begin{table*}
    \footnotesize
    \centering
    \begin{tabularx}{\textwidth}{Xcccccc}
        \thickhline
        Model & Spatial comp & Temporal comp (short) & Long-temporal & F1@10 & Edit & Acc  \\
        \hline \hline
        SpatialCNN \cite{lea2016segmental}   & RGB & MHI & -          & 41.8 & - & 54.1 \\
        (SpatialCNN) + ST-CNN \cite{lea2016segmental}       & RGB & MHI & 1D-Conv    & 58.7 & - & 60.6 \\
        (SpatialCNN) + DilatedTCN \cite{lea2017tempconvnet} & RGB & MHI & DilatedTCN & 58.8 & - & 58.3 \\
        (SpatialCNN) + ED-TCN \cite{lea2017tempconvnet}     & RGB & MHI & ED-TCN     & 72.2 & - & 64.0 \\
        (SpatialCNN) + TDRN \cite{lei2018temporal}          & RGB & MHI & TDRN       & (79.2) & (74.1) & (70.1) \\
        \hline
        LCDC        & RGB & Learned deformation & -       & 52.42 & 45.38 & 55.32 \\
        LCDC + ST-CNN & RGB & Learned deformation & 1D-Conv & 62.23$\pm$0.69 & 55.75$\pm$0.94 & 58.36$\pm$0.45 \\
        LCDC + DilatedTCN & RGB & Learned deformation & DilatedTCN & 62.08$\pm$0.85 & 55.13$\pm$0.79 & 58.07$\pm$0.30 \\
        LCDC + ED-TCN & RGB & Learned deformation & ED-TCN  & \textbf{75.39}$\pm$\textbf{1.33} & \textbf{72.84}$\pm$\textbf{0.84} & \textbf{65.34}$\pm$\textbf{0.54} \\
        \thickhline
    \end{tabularx}
    \vspace{-2mm}
    \caption{Results on GTEA dataset. Table format follows the same convention as in \tabref{tab:results_50_salads}.}
    \label{tab:results_gtea}
    \vspace{-4mm}
\end{table*}

\begin{figure*}
    \centering
    \begin{subfigure}[t]{0.50\linewidth}
		\includegraphics[width=\linewidth, height=7.5em]{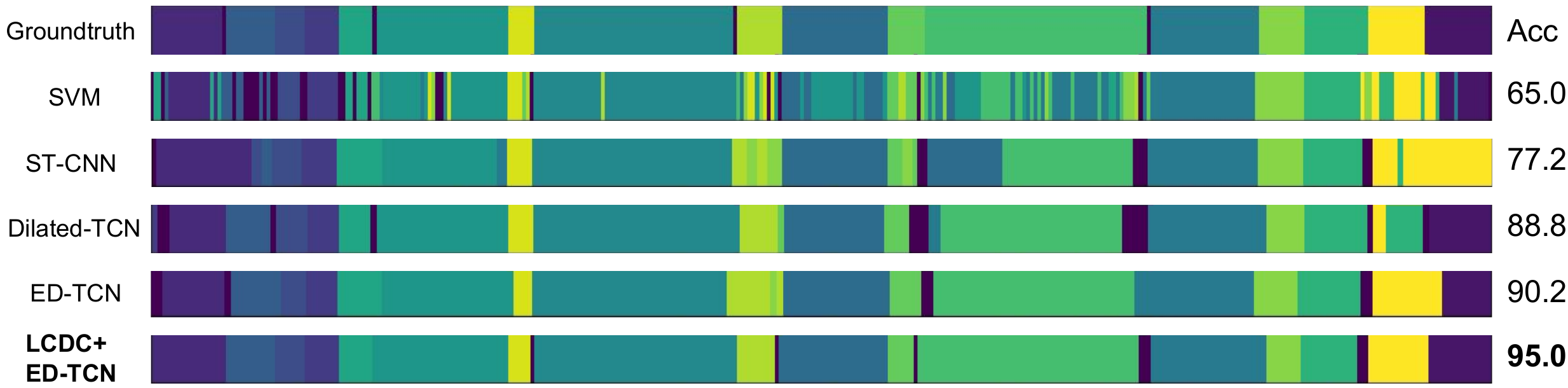}
		\caption{50 Salads dataset (\textit{mid}-level).}
		\label{sfig:compare_segments_50salads}
	\end{subfigure}
	\hfill
    \begin{subfigure}[t]{0.45\linewidth}
		\includegraphics[width=\linewidth, height=7.5em, trim={2.3cm 0 0 0},clip]{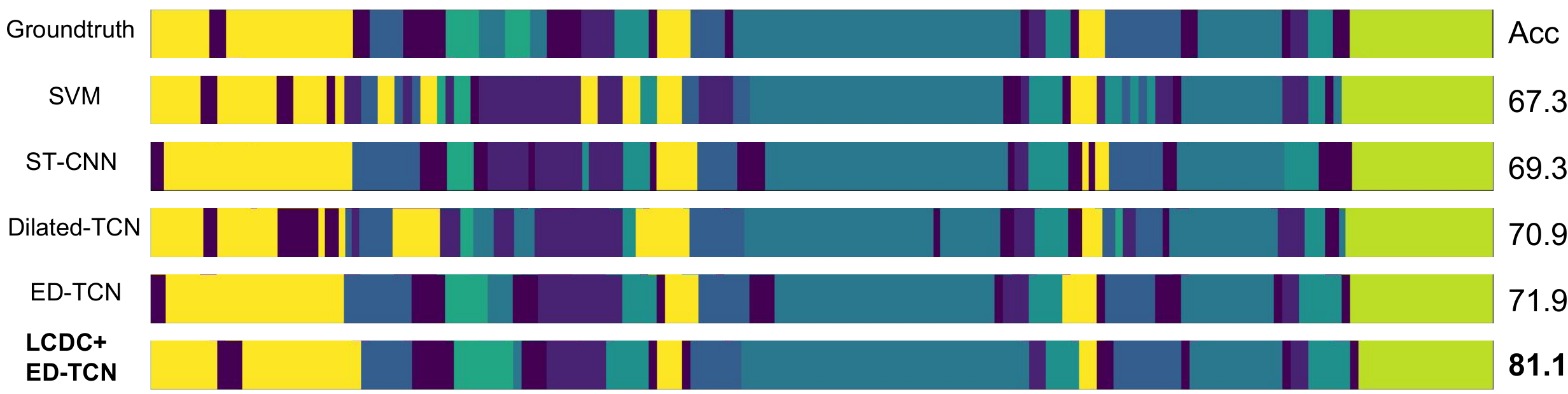}
		\caption{GTEA dataset.}
		\label{sfig:compare_segments_gtea}
	\end{subfigure}
	\vspace{-2mm}
    \caption{Comparison of segmentation results across different methods on two test videos (one each for 50 Salads and GTEA dataset). SVM, ST-CNN, DilatedTCN, and ED-TCN are original results with SpatialCNN features. LCDC features are used in conjunction with ED-TCN long-temporal model in the last row. Framewise accuracy is reported for each setup.}
    \label{fig:compare_segments}
\end{figure*}

\begin{table*}
    \footnotesize
    \centering
    \begin{tabularx}{\textwidth}{Xcccccccc}
        \thickhline
        Model & Spatial comp & Temporal comp (short) & Fusion scheme & Acc & Total params & Deform params \\
        \hline \hline
        SpatialCNN          & RGB (single) & MHI (multi) & Stacked inputs  & 60.99 & -     & - \\
        NaiveAppear     & RGB (single) & -           & -               & 68.45 & 38.9M & - \\
        NaiveTempAppear & RGB (multi)  & Avg feat frames (multi) & -   & 71.52 & 38.9M & - \\
        OptFlowMotion       & -            & OptFlow (multi) & -           & 25.67 & 134.1M & - \\
        TwoStreamNet        & RGB (multi)  & OptFlow (multi) & Avg scores  & 71.82 & 173.0M & - \\
        \hline
        DC   & RGB (multi) & Learned deformation (w/o local coherency) (multi) & 3D-Conv & 72.25 & 45.7M & 995.5K \\
        LCDC & RGB (multi) & Learned deformation (multi) & 3D-Conv & \textbf{73.77} & \textbf{42.7M} & \textbf{27.7K} \\
        \thickhline
    \end{tabularx}
    \vspace{-2mm}
    \caption{Ablation study on 50 Salads dataset (Split 1, \textit{mid}-level). ``Single'' and ``multi'' indicate the amount of input frames for spatial/temporal components.}
    \label{tab:ablation}
    \vspace{-4mm}
\end{table*}

We benchmark our approach using three standard metrics reported in \cite{lea2017tempconvnet, lei2018temporal}: frame-wise accuracy, segmental edit score, and F1 score with overlapping of 10\% (F1@10). Since edit and F1 scores penalize over-segmentation, accuracy metric is more suitable to evaluate the quality of short-term spatio-temporal features (SpatialCNN and LCDC). All mentioned metrics are sufficient to assess the performance of long-temporal models (ST-CNN, DilatedTCN, ED-TCN, and TDRN). We have also specified inputs for spatial and short-term temporal components, as well as the long-temporal model in each setup (\tabref{tab:results_50_salads} and \tabref{tab:results_gtea}).

\tabref{tab:results_50_salads} shows the results on 50 Salads dataset on both granularity levels. Overall performance of LCDC setups, with long-temporal models, outperform their counterparts. We highlight our LCDC + ED-TCN setups as they provided the most significant improvement over other baselines. Compared to the original ED-TCN, which used SpatialCNN features, our approach increases by 5.75\%, 7.14\%, 7.42\% on mid-level and 3.72\%, 2.36\%, 5.5\% on eval-level, in terms of F1@10, edit score, and accuracy. 
\tabref{tab:results_gtea} shows the results on GTEA dataset and is organized in a fashion similar to \tabref{tab:results_50_salads}. 
We achieve the best performance when incorporating LCDC features with ED-TCN framework out of the three baselines. LCDC + ED-TCN also outperforms the original SpatialCNN + ED-TCN on both reported metrics: improving by 3.19\% and 1.34\%, in terms of F1@10 and accuracy.

We further show segmentation results of test videos from 50 Salads (on \textit{mid}-level granularity) (\figref{sfig:compare_segments_50salads}) and GTEA datasets (\figref{sfig:compare_segments_gtea}). In the figures, the first row is the ground-truth segmentation. The next four rows are results from different long-temporal models using SpatialCNN features: SVM, ST-CNN, DilatedTCN, and ED-TCN. All of these segmentation results are directly retrieved from the provided features in \cite{lea2017tempconvnet}, without any further training. The last row shows the segmentation results of our LCDC + ED-TCN. Each row also comes with its respective accuracy on the right. On 50 Salads dataset, \figref{sfig:compare_segments_50salads} shows that LCDC + ED-TCN achieves a 4.8\% improvement over original ED-TCN. On GTEA dataset, \figref{sfig:compare_segments_gtea} shows a strong improvement of LCDC over ED-TCN, being 9.2\% in terms of accuracy. We also achieve a higher accuracy on the temporal boundaries, \ie the beginning and the end of an action instance is close to that of ground-truth.

\subsection{Ablation study}
\label{ssec:ablation}

We performed an ablation study (\tabref{tab:ablation}) on Split 1 and \textit{mid}-level granularity of 50 Salads dataset to compare LCDC with SpatialCNN and a two-stream framework. For each setup (each row in the table), we show the inputs for spatial and short-term temporal components, its fusion scheme, frame-wise accuracy, the total number of parameters of the model, and the number of parameters related to deformable convolutions (wherever applicable). Since this experiment focuses on comparing short-term features, accuracy metric is more suitable. We also report whether a component requires single or multiple frames as input.

We evaluate on the following setups: 
\textbf{(1) SpatialCNN:} The features from \cite{lea2016segmental} described in Section \ref{ssec:baselines}. Its inputs are stacked RGB frame and MHI. 
\textbf{(2) NaiveAppear:} Frame-wise class prediction using ResNet50 (no temporal information involved in this setup). 
\textbf{(3) NaiveTempAppear:} Appearance stream from conventional two-stream frameworks uses a single frame input and VGG backbone. Therefore, comparing LCDC with the above is not straight-forward. We created an appearance stream with multiple input frames and ResNet50 backbone for better comparison with LCDC. Temporal component was modeled by averaging feature frames (before feeding to two fc layers with ReLU). This model is the same as \textit{NaiveAppear}, except that we have multiple frames per video snippet. 
\textbf{(4) OptFlowMotion:} Motion stream that models temporal component using VGG-16 (with stacked dense optical flows as input). This is similar to the motion component of conventional two-stream networks.
\textbf{(5) TwoStreamNet:} The two-stream framework obtained by averaging scores from \textit{NaiveTempAppear} and \textit{OptFlowMotion}. We follow the fusion scheme used in conventional two-stream network \cite{simonyan14twostream}.
\textbf{(6) DC:} Receptive fields of deformable convolution network (with backbone ResNet50) are used to model motion, but without local coherency constraint. 
\textbf{(7) LCDC:} The proposed LCDC model which additionally enforces local coherency constraint on receptive fields.

Compared to \textit{SpatialCNN}, \textit{NaiveAppear} has a higher accuracy because the \textit{SpatialCNN} features are extracted using VGG-like model while \textit{NaiveAppear} uses ResNet50. The accuracy is further improved by 3.07\% by averaging multiple feature frames in \textit{NaiveTempAppear}. Notice that the number of parameters of \textit{NaiveAppear} and \textit{NaiveTempAppear} are the same because the only difference is the number of frames being used as input (averaging requires no parameters). Accuracy from \textit{OptFlowMotion} is lower than other models because the motion in 50Salads is hard to capture using optical flow. This is consistent with the observation in \cite{lea2017tempconvnet, lea2016segmental} that optical flow is inefficient for the dataset. Combining \textit{OptFlowMotion} with \textit{NaiveTempAppear} in \textit{TwoStreamNet} slightly improves the performance. However, the number of parameters is significantly increased because of complexity of \textit{OptFlowMotion}. This prevented us from having a larger batch size or training the two streams together.

Both of our DC and LCDC frameworks, which model temporal components as difference of receptive fields, outperform the two-stream approach \textit{TwoStreamNet} with significantly lower model complexities. \textit{DC}, which directly uses adaptive receptive fields from the original deformable convolution, increases the accuracy to 72.25\%. LCDC further improves accuracy to 73.77\% and with even fewer parameters. This complexity reduction is because LCDC uses fewer parameters for deformation offsets. It means the extra parameters of DC are not necessary to model spatio-temporal features, and thus can be removed. Moreover, if we consider only the parameters related to deformable convolutions, \textit{DC} would require 36x more parameters than \textit{LCDC}. The reduction of 36x matches our derivation in Sec \ref{ssec:dconv_motion_constraint}, where $K_h$=$K_w$=$3$ and $G$=$4$. The number of reduced parameters is proportional to the number of deformable convolution layers.
\section{Conclusion}
\label{sec:conclusion}
We introduced locally-consistent deformable convolution (LCDC) and created a single-stream network that can jointly learn spatio-temporal features by exploiting motion in adaptive receptive fields. The framework is significantly more compact and can produce robust spatio-temporal features without using conventional motion extraction methods, \eg optical flow. LCDC features, when incorporated into several long-temporal networks, outperformed their original implementations. For future work, we plan to unify long-temporal modeling directly into the framework.
\\
{
\noindent\textbf{Acknowledgments:} 
This material is based upon work supported in part by C3SR. Rogerio Feris is partly supported by IARPA via DOI/IBC contract number D17PC00341.
The U.S. Government is authorized to reproduce and distribute reprints for Governmental purposes notwithstanding any copyright annotation thereon. Disclaimer: The views and conclusions contained herein are those of the authors and should not be interpreted as necessarily representing the official policies or endorsements, either expressed or implied, of IARPA, DOI/IBC, or the U.S. Government)
}

{\small
\bibliographystyle{ieee_fullname}
\bibliography{ref}
}

\onecolumn

\appendix
\begin{center}
    \Large
    \textbf{Supplementary Material for Learning Motion in Feature Space: Locally-Consistent Deformable Convolution Networks for Fine-Grained Action Detection}
\end{center}
\addcontentsline{toc}{section}{Appendices}
\renewcommand{\thesubsection}{\Alph{subsection}}

\section{Full formulation for convolutions with multiple output channels}
Suppose that the input $\bx$ of a convolution has $I$ channels and the output has $O$ channels, \ie $\bx \in R^{N \times I}, \by \in \R^{M \times O}$, we can write the standard convolution as:
\begin{equation}
    \by_j[n] = \sum_i \sum_k \bw_{j,i}[-k] \bx_i \left[ n+k \right],
\end{equation}
where $i \in \{1, \dots, I\}, j \in \{1, \dots, O\}$, and $\bw \in \R^{N \times K \times I \times O}$. The original deformable convolution, therefore, is written as:
\begin{equation}
    \by_j[n] = \sum_i \sum_k \bw_{j,i}[-k] \bx_i \left( n + k + \ddot{\Delta}_{n, k} \right),
\end{equation}
In reality, there are multiple deformable groups ($G > 1$), meaning that different input channels can have different deformation offsets. Specifically, a multi-channel deformable convolution with multiple deformable group can be written as:
\begin{equation}
    \by_j[n] = \sum_i \sum_k \bw_{j,i}[-k] \bx_i \left( n + k + \ddot{\Delta}_{g_i, n, k} \right),
\end{equation}
where $g_i$ is the deformable group that the input channel $i$ belongs to. We keep the deformable group as $G=1$ and drop the notation $g_i$ for the sake of simplicity.

We write the multi-channel LCDC as:
\begin{equation}
    \by_j[n] = \sum_i \sum_k \bw_{j,i}[-k] \bx_i \left( n + k + \dot{\Delta}_{n+k} \right).
\end{equation}
It is equivalent to
\begin{equation}
    \by_j[n] = \sum_i \sum_k \bw_{j,i}[-k] \tilde{\bx}_i[n+k] = (\tilde{\bx}*\bw_{j})[n],
\end{equation}
where
\begin{equation}
    \tilde{\bx}_i[n] = (D_{\dot{\Delta}} \{\bx_i\})[n] = \bx_i\left( n + \dot{\Delta}_n \right).
\end{equation}

\section{More reasoning on the difference of receptive fields}
$\ddot{\br}^{(t)}$ and $\dot{\br}^{(t)}$ of deformable convolution and locally-consistent deformable convolution \textit{carries temporal information} because the offsets are constructed from inputs at different time frames. This property is not valid in other types of convolutions. We can write standard convolutions and dilated convolutions as special cases of deformable convolutions, \ie $\ddot{\Delta} = 0$ in standard convolution and $\ddot{\Delta}^{(t)} = const, \forall t$. Hence,
\begin{itemize}
    \item Standard convolution: \[\ddot{\Delta}^{(t)} = 0, \forall t \Rightarrow \ddot{\br}^{(t)} = 0, \forall t,\]
    \item Dilated convolution: \[\ddot{\Delta}^{(t)} = \ddot{\Delta}^{(t-1)}, \forall t \Rightarrow \ddot{\br}^{(t)} = 0, \forall t.\]
    \item Deformable convolution: \[\ddot{\Delta}^{(t)} \neq \ddot{\Delta}^{(t-1)} \Rightarrow \ddot{\br}^{(t)} \neq 0.\]
    \item Locally-consistent deformable convolution: \[\dot{\Delta}^{(t)} \neq \dot{\Delta}^{(t-1)} \Rightarrow \dot{\br}^{(t)} \neq 0.\]
\end{itemize}

\section{In-detail architecture of LCDC}
\tabref{tab:lcdc_archi_detail} shows the detailed architecture implementation of LCDC.

{\footnotesize
\begin{longtable}{l|llll}
    \centering
    Layer & Input(s) & Output size & Kernel size & Comments \\
    \hline \hline
    conv1  & data      & (112,112,64) & (7,7,64), stride2 & with 3,3 maxpool, stride2 \\
           &           &              &                   & frames of all snippets \\
           &           &              &                   & are unrolled \\
    \hline
    conv2x & bn\_conv1 & (56,56,256) & $\begin{bmatrix}1,1,64\\3,3,64\\1,1,256\end{bmatrix}\times 3$ & input is output of conv1 \\
    \hline
    conv3x & res2c\_relu & (28,28,512) & $\begin{bmatrix}1,1,128\\3,3,128\\1,1,512\end{bmatrix}\times 4$ & input is output of conv2x \\
    \hline
    conv4x & res3d\_relu & (14,14,1024) & $\begin{bmatrix}1,1,256\\3,3,256\\1,1,1024\end{bmatrix}\times 6$ & input is output of conv3x \\
    \hline
    res5a\_branch1                  & res4f\_relu           & (14,14,2048) & (1,1,2048) & input is output of conv4x \\
    bn5a\_branch1                   & (prev)                & (14,14,2048) & -          & batch normalization \\
    res5a\_branch2a                 & (prev)                & (14,14,512)  & (1,1,512)  & convolution \\
    bn5a\_branch2a                  & (prev)                & (14,14,512)  & -          & batch normalization \\
    res5a\_branch2a\_relu           & (prev)                & (14,14,512)  & -          & ReLU \\
    res5a\_branch2b\_offset         & (prev)                & (14,14,2)    & (3,3,2)    & offset learner \\
    res5a\_branch2b\_offset\_expand & (prev)                & (14,14,18)   & -          & expand by replication \\
    res5a\_branch2b                 & res5a\_branch2a\_relu & (14,14,512)  & (3,3,512)  & deformable convolution \\
                                    & res5a\_branch2b\_offset\_expand & & & \\
    bn5a\_branch2b                  & (prev)                & (14,14,512)  & -          & batch normalization \\
    res5a\_branch2b\_relu           & (prev)                & (14,14,512)  & -          & ReLU \\
    res5a\_branch2c                 & (prev)                & (14,14,2048) & (1,1,2048) & convolution\\
    bn5a\_branch2c                  & (prev)                & (14,14,2048) & -          & batch normalization \\
    res5a                           & bn5a\_branch1         & (14,14,2048) & -          & addition \\
                                    & bn5a\_branch2c        & & & \\
    res5a\_relu                     & (prev)                & (14,14,2048) & -          & ReLU \\
    \hline
    res5b\_branch2a                 & (prev)                & (14,14,512)  & (1,1,512)  & convolution \\
    bn5b\_branch2a                  & (prev)                & (14,14,512)  & -          & batch normalization \\
    res5b\_branch2a\_relu           & (prev)                & (14,14,512)  & -          & ReLU \\
    res5b\_branch2b\_offset         & (prev)                & (14,14,2)    & (3,3,2)    & offset learner \\
    res5b\_branch2b\_offset\_expand & (prev)                & (14,14,18)   & -          & expand by replication \\
    res5b\_branch2b                 & res5b\_branch2a\_relu & (14,14,512)  & (3,3,512)  & deformable convolution \\
                                    & res5b\_branch2b\_offset\_expand & & & \\
    bn5b\_branch2b                  & (prev)                & (14,14,512)  & -          & batch normalization \\
    res5b\_branch2b\_relu           & (prev)                & (14,14,512)  & -          & ReLU \\
    res5b\_branch2c                 & (prev)                & (14,14,2048) & (1,1,2048) & convolution \\
    bn5b\_branch2c                  & (prev)                & (14,14,2048) & -          & batch normalization \\
    res5b                           & res5a\_relu           & (14,14,2048) & -          & addition \\
                                    & bn5b\_branch2c        & & & \\
    res5b\_relu                     & (prev)                & (14,14,2048) & -          & ReLU \\
    \hline
    res5c\_branch2a                 & (prev)                & (14,14,512)  & (1,1,512)  & convolution \\
    bn5c\_branch2a                  & (prev)                & (14,14,512)  & -          & batch normalization \\
    res5c\_branch2a\_relu           & (prev)                & (14,14,512)  & -          & ReLU \\
    res5c\_branch2b\_offset         & (prev)                & (14,14,2)    & (3,3,2)    & offset learner \\
    res5c\_branch2b\_offset\_expand & (prev)                & (14,14,18)   & -          & expand by replication \\
    res5c\_branch2b                 & res5c\_branch2a\_relu & (14,14,512)  & (3,3,512)  & deformable convolution \\
                                    & res5c\_branch2b\_offset\_expand & & & \\
    bn5c\_branch2b                  & (prev)                & (14,14,512)  & -          & batch normalization \\
    res5c\_branch2b\_relu           & (prev)                & (14,14,512)  & -          & ReLU \\
    res5c\_branch2c                 & (prev)                & (14,14,2048) & (1,1,2048) & convolution \\
    bn5c\_branch2c                  & (prev)                & (14,14,2048) & -          & batch normalization \\
    res5c                           & res5b\_relu           & (14,14,2048) & -          & addition \\
                                    & bn5c\_branch2c        & & & \\
    res5c\_relu                     & (prev)                & (14,14,2048) & -          & ReLU \\
    conv\_new\_1                    & (prev)                & (14,14,256)  & (1,1,256)  & convolution \\
    conv\_new\_1\_relu              & (prev)                & (14,14,256)  & -          & ReLU \\
    \hline
    spacetime\_fusion & conv\_new\_1\_relu      & (L-1,14,14,262) & -           & reshape all frames back into \\
                      & res5a\_branch2b\_offset &                 &             & snippets, then concatenate \\
                      & res5b\_branch2b\_offset &                 &             & difference of all offset layers \\
                      & res5c\_branch2b\_offset &                 &             & with conv\_new\_1\_relu \\
    spacetime\_conv1  & (prev)                  & (L-1,14,14,256) & (4,3,3,256) & 3Dconv with window size \\
                      &                         &                 &             & for temporal dimension of 4\\
    spacetime\_bn1    & (prev)                  & (L-1,14,14,256) & -           & batch normalization \\
    spacetime\_relu1  & (prev)                  &                 & -           & ReLU \\
    spacetime\_pool1  & (prev)                  & ((L-1)/2,14,14,256) & -       & temporal max pooling of size 2\\
    spacetime\_conv2  & (prev)                  & ((L-1)/2,14,14,256) & (4,3,3,256) & 3Dconv with window size \\
                      &                         &                     &             & for temporal dimension of 4\\
    spacetime\_bn2    & (prev)                  & ((L-1)/2,14,14,256) & -           & batch normalization \\
    spacetime\_relu2  & (prev)                  &                     & -           & ReLU \\
    spacetime\_pool2  & (prev)                  & ((L-1)/4,14,14,256) & -           & temporal max pooling of size 2\\
    spacetime\_reduce & (prev)                  & (14,14,256)         & -           & averaging across time domain \\
    \hline
    pool\_new          & (prev) & (7,7,256) & - & max pooling, stride 2\\
    fc\_new\_1         & (prev) & (1024)    & - & fully connected with ReLU \\
    fc\_new\_2         & (prev) & (1024)    & - & fully connected with ReLU \\
    \caption{LCDC architecture in detail. The groups conv1, conv2x, conv3x, and conv4x are the same as the original ResNet50. The convention of kernel size: (kernel\_height, kernel\_width, number\_of\_output\_channels) for 2D convolution and (kernel\_time, kernel\_height, kernel\_width, number\_of\_output\_channels) for 3D convolution. Size of input data is (224, 224, 3). $L$ is the number of frames per video snippet (we choose $L=16$). If the input is annotated as \textit{(prev)}, it means it uses the output from the previous layer.}
    \label{tab:lcdc_archi_detail}
\end{longtable}}

\section{In-detail figures}
We provide higher-resolution versions of \figref{fig:archi_overview}, \figref{fig:mrDeRF_archi}, and \figref{fig:receptive_fields} in \figref{fig:archi_overview_big}, \figref{fig:mrDeRF_archi_big}, and \figref{fig:receptive_fields_big} respectively. 
\figref{fig:compare_segments_50salads_big} and \figref{fig:compare_segments_gtea_big} also show higher-resolution versions of \figref{fig:compare_segments} with annotation of color-code. We also provide the groundtruth action sequence of the two videos. Readers can view the videos corresponding to \figref{fig:compare_segments_50salads_big} and \figref{fig:compare_segments_gtea_big} in other additional supplementary materials (\textit{50salads.mp4} and \textit{gtea.mp4}).

\begin{figure}[H]
	\centering
	\includegraphics[width=0.7\linewidth]{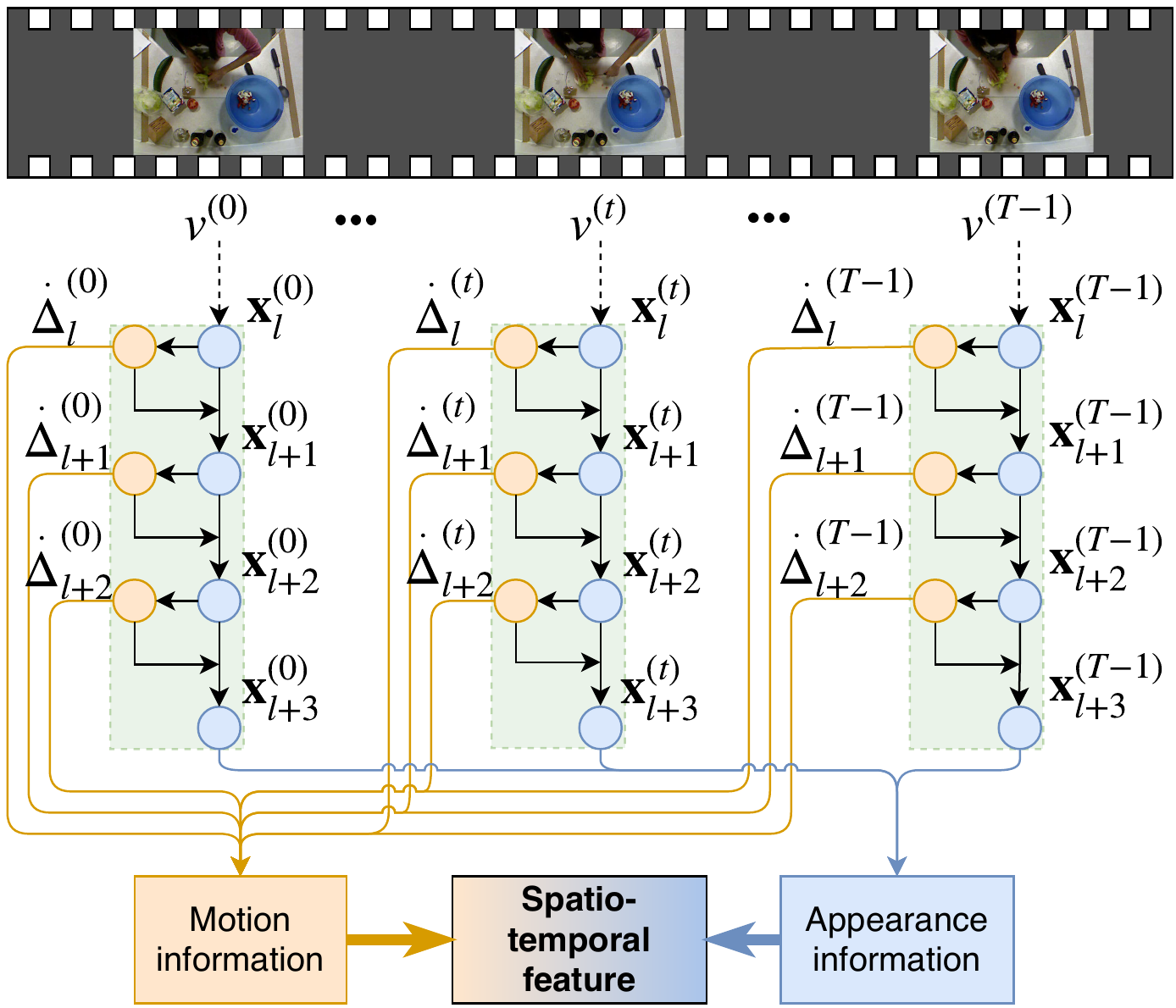}
	\caption{Network architecture of our proposed framework across multiple frames $v^{(t)}$. Appearance information comes from the last layer while motion information is extracted directly from deformation $\dot{\Delta}$ in the feature space instead of from a separate optical flow stream. Weights are shared across frames over time.}
	\label{fig:archi_overview_big}
\end{figure}

\begin{figure}
	\centering
	\includegraphics[width=\linewidth]{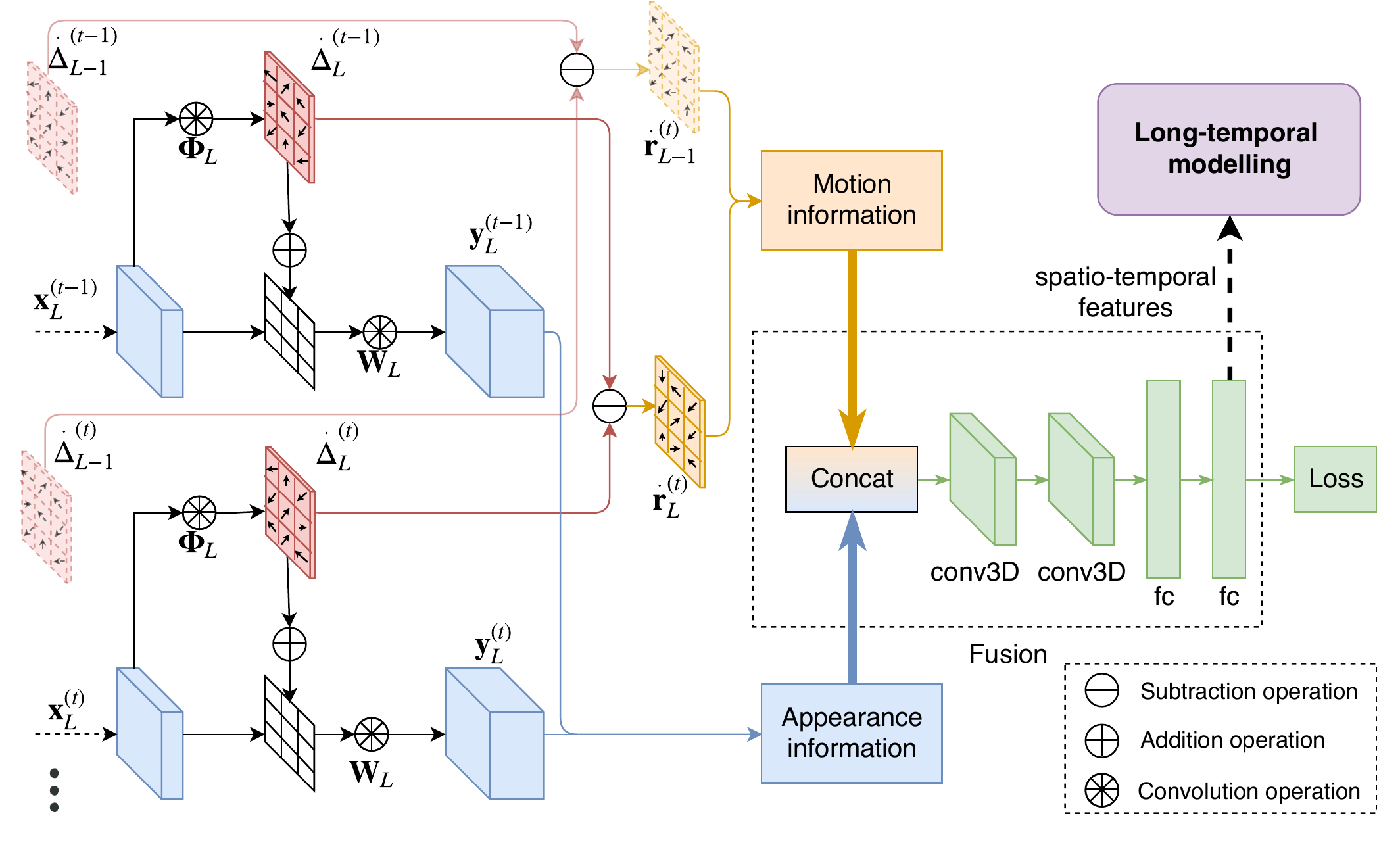}
	\caption{A more detailed view of our network architecture with the fusion module. Appearance information comes from output of the last layer while motion information comes from aggregating $\dot{\br}$ from multiple layers. Outputs of the final fc layer can be flexibly used as the features for any long-temporal modeling networks.}
	\label{fig:mrDeRF_archi_big}
\end{figure}

\begin{figure}
    \centering
    \includegraphics[width=0.9\linewidth]{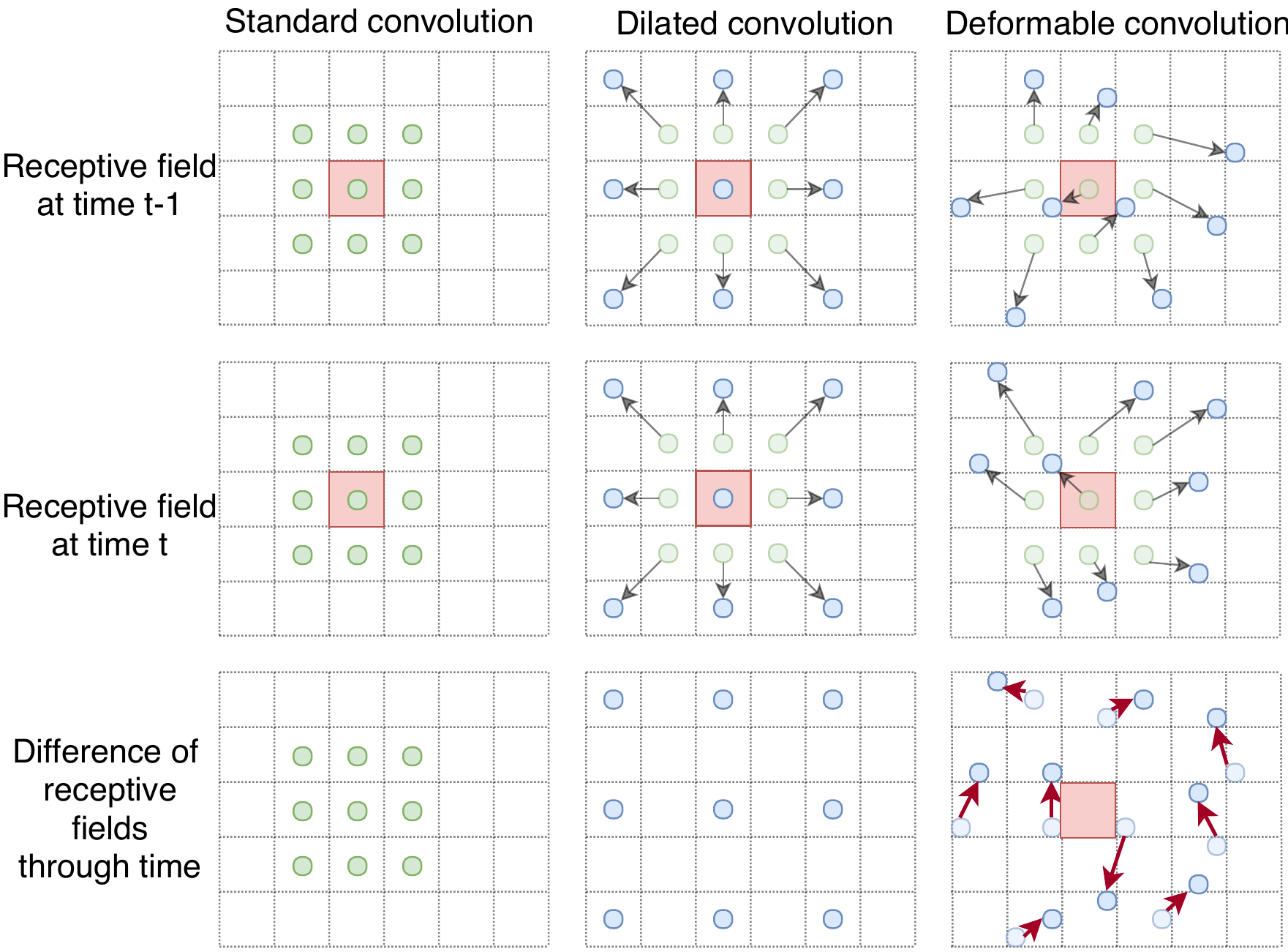}
    \caption{Illustration of temporal information modeled by the difference of receptive fields at a single location in 2D. Only deformable convolution can capture temporal information (shown with red arrows). Related to \equref{equ:deconv} and \equref{equ:diff_receptive_field}, $n$ is red square, $n+k$ are green dots, $\ddot{\Delta}_{n,k}$ are black arrows, $n+k+\ddot{\Delta}_{n,k}$ are blue dots, and $\ddot{\br}$ are red arrows.}
    \label{fig:receptive_fields_big}
\end{figure}

\begin{figure}
    \centering
    \includegraphics[width=0.9\linewidth]{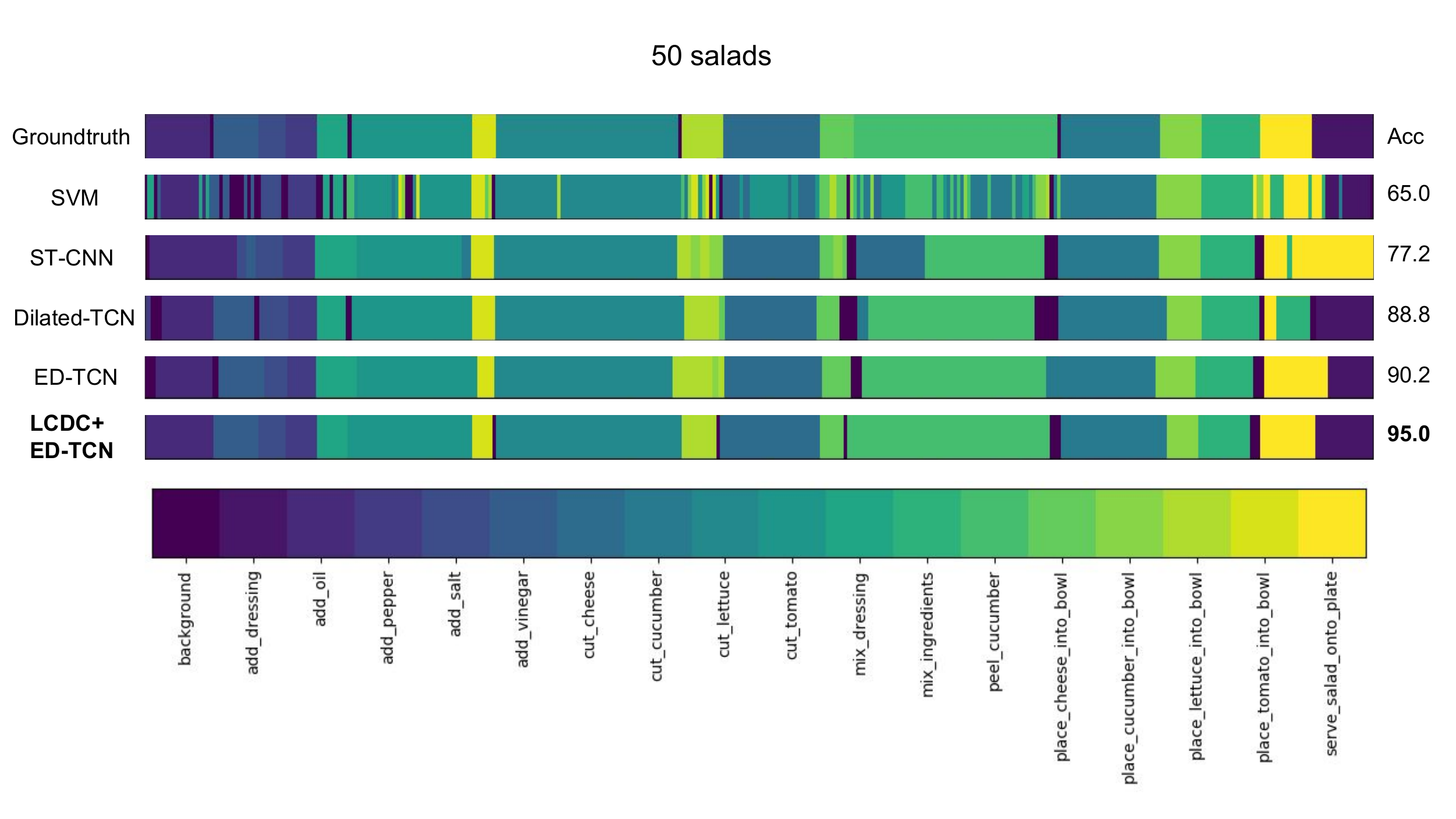}
    \vspace{-2mm}
    \caption{Comparison of segmentation results across different methods on a test video from 50 Salads dataset (\textit{mid}-level). The action sequence is: \textit{add\_oil, background, add\_vinegar, add\_salt, add\_pepper, mix\_dressing, background, cut\_tomato, place\_tomato\_into\_bowl, cut\_lettuce, background, place\_lettuce\_into\_bowl, cut\_cheese, place\_cheese\_into\_bowl, peel\_cucumber, background, cut\_cucumber, place\_cucumber\_into\_bowl, mix\_ingredients, serve\_salad\_onto\_plate, add\_dressing}.}
    \label{fig:compare_segments_50salads_big}
\end{figure}

\begin{figure}
    \centering
    \includegraphics[width=0.9\linewidth]{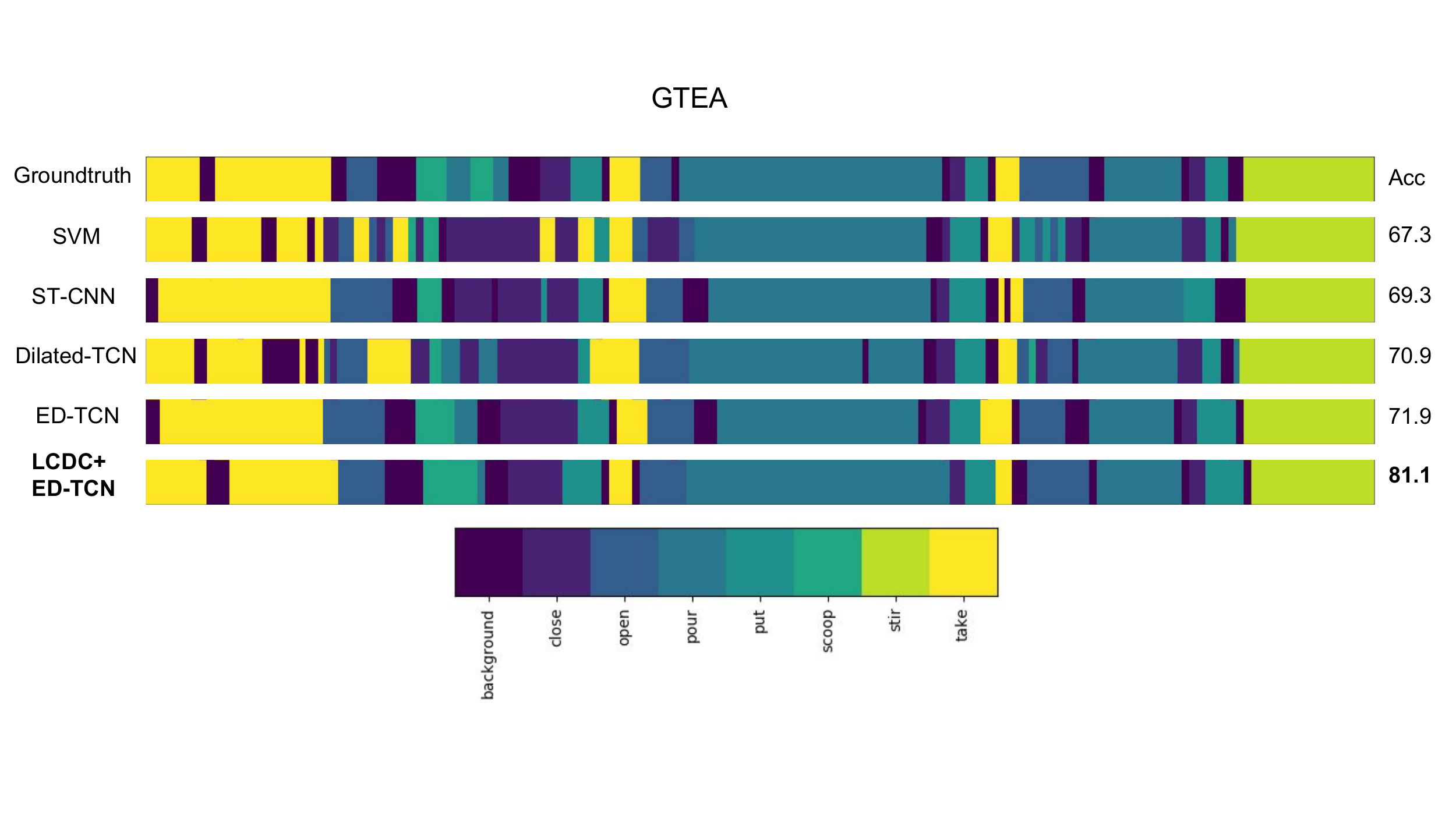}
    \vspace{-2mm}
    \caption{Comparison of segmentation results across different methods on a test video from GTEA dataset. The action sequence is: \textit{take, background, take, background, open, background, scoop, pour, scoop, pour, background, close, put, background, take, open, background, pour, background, close, put, background, take, open, background, pour, background, close, put, background, stir}.}
    \label{fig:compare_segments_gtea_big}
\end{figure}

\end{document}